\documentclass[conference]{IEEEtran}
\IEEEoverridecommandlockouts
\usepackage{graphicx} 

\usepackage{amsmath}
\usepackage{amsthm}
\usepackage{amssymb}
\usepackage{mathtools}
\usepackage[T1]{fontenc}
\usepackage[utf8]{inputenc}
\usepackage{subfig}
\usepackage{color}
\usepackage{cite}
\usepackage{bbm} 
\usepackage{enumitem} 
\usepackage{algorithm}
\usepackage{algpseudocode}
\usepackage{multirow}
\usepackage{mathdots}
\usepackage[hyperfootnotes=false,hidelinks,bookmarks,bookmarksdepth=1,bookmarksnumbered]{hyperref}

\DeclareMathOperator*{\argmin}{arg\,min}

\newtheorem{lemma}{Lemma}

\newtheorem{remark}{Remark}

\def\BibTeX{{\rm B\kern-.05em{\sc i\kern-.025em b}\kern-.08em
    T\kern-.1667em\lower.7ex\hbox{E}\kern-.125emX}}

\begin{document}

\title{PtyGenography:  using generative models for regularization of the phase retrieval problem
\thanks{The authors would like to thank the Lorentz Center for their support during PRiMA workshop, where this project has been initiated. PS is supported by NWO Talent program Veni ENW grant, file number VI.Veni.212.176.}
}

\author{\IEEEauthorblockN{Selin Aslan}
\IEEEauthorblockA{\textit{Department of Mathematics} \\
\textit{Ko\c{c} University}\\
saslan@ku.edu.tr}
\and
\IEEEauthorblockN{Tristan van Leeuwen}
\IEEEauthorblockA{\textit{Computational Imaging} \\
\textit{Centrum Wiskunde \& Informatica (CWI)}\\
t.van.Leeuwen@cwi.nl}
\and
\IEEEauthorblockN{Allard Mosk}
\IEEEauthorblockA{\textit{Debye Institute}\\
\textit{\& Department of Physics} \\
\textit{Utrecht University}\\
a.p.mosk@uu.nl}
\and
\IEEEauthorblockN{Palina Salanevich}
\IEEEauthorblockA{\textit{Mathematical Institute} \\
\textit{Utrecht University}\\
p.salanevich@uu.nl}
}

\maketitle

\begin{abstract}
In phase retrieval and similar inverse problems, the stability of solutions across different noise levels is crucial for applications. One approach to promote it is using signal priors in a form of a generative model as a regularization, at the expense of introducing a bias in the reconstruction. In this paper, we explore and compare the reconstruction properties of classical and generative inverse problem formulations. We propose a new unified reconstruction approach that mitigates overfitting to the generative model for varying noise levels.
\end{abstract}


\section{Introduction}

In many signal processing applications in imaging, acoustics, and optics, we aim to reconstruct an unknown signal from its intensity measurements. This inverse problem is known as the \emph{phase retrieval problem}. In some applications, such as diffraction imaging \cite{millane1990phase, bunk2007diffractive} and ptychography \cite{rodenburg2008ptychography}, it arises due to the limitations of the measurement process. Namely, the detectors are unable to capture the phases of the measurements, but only their magnitudes. In other applications, including audio processing \cite{rabiner1999fundamentals, balan2006signal, schulze2021sparse}, the phases of the measurements may be too noisy to be used in the reconstruction process. 

Current technological advances and the pressing need for fast, high-fidelity methods in imaging and audio processing inspired active interest in phase retrieval. In recent years, substantial progress has been made in the mathematical foundations of this problem~\cite{Grohs2020, Dong2023}. However, there is still a significant gap between the available theoretical reconstruction guarantees and the application restrictions and requirements. As such, it is often desired to achieve stable reconstruction while facing high measurement noise and a limited measurement budget that is smaller than the minimal number of measurements theoretically required for reconstruction~\cite{heinosaari2013quantum}.

More formally, we consider a non-linear inverse problem of retrieving ${\bf f}\in \mathbb{C}^n$ from noisy measurements ${\bf y} = \mathcal{A}({\bf f}) + \boldsymbol{\varepsilon}$, where $\mathcal{A}\colon \mathbb{C}^n \to \mathbb{C}^m$ is the measurement map, and $\boldsymbol{\varepsilon}\in \mathbb{C}^m$ is the noise term.
In the case of the phase retrieval problem, the measurement map is defined as $\mathcal{A}({\bf f}) = \vert A {\bf f} \vert^2$, where ${A\in\mathbb{C}^{m\times n}}$ is a given measurement matrix and $|\cdot|^2$ denotes the element-wise squared absolute value of a vector. Note that in this paper, we only consider the additive noise model, which is typical for detector noise and does not cover Poisson shot noise.
Assuming $\mathcal{A}$ is injective, one can formulate the reconstruction problem as
\begin{equation}\label{eq: conventional}
\min_{{\bf f}} \|\mathcal{A}({\bf f})-{\bf y}\|_2^2.
\end{equation}
To account for instabilities in the presence of noise, one often adds a regularization term. The most common classical regularization approach is Tikhonov regularization, which is theoretically well-understood~\cite{vasin2006some}. 
However, it tends to smooth high-frequency components, which can be problematic when trying to capture detailed signal features.

An alternative budding approach to enforce stability in the presence of noise is incorporating the available prior information on the signal into the reconstruction process. This reduces the set of admissible signals and thus can drastically improve reconstruction accuracy and the number of phaseless measurements required for reconstruction.

Inspired by the advances in the field of compressive sensing, various works have explored sparsity priors on the signal~\cite{eldar2014sparse, jaganathan2017sparse, Grohs2020}. However, while in many applications, the set of admissible signals has a lower-dimensional structure, it is not always the case that it can be approximated by a sparse representation. Following the recent remarkable success of neural networks in phase retrieval applications~\cite{deng2020learning, rivenson2018phase, sinha2017lensless}, a new framework for modeling the set of admissible signals as the range $\{{\bf f} = \mathcal{G}({\bf z})\colon {\bf z}\in \mathbb{R}^k\}$ of a (deep) generative neural network~$\mathcal{G}\colon \mathbb{C}^k \to \mathbb{C}^n$ with $k\ll n$ was proposed in~\cite{hand2018phase}. In this paper, the authors studied the case of a generic measurement frame and an untrained neural network~$\mathcal{G}$ with random parameters. In the case of structured, application-relevant measurements, the injectivity of phase retrieval from power spectra (squared absolute values of the oversampled Fourier measurements) was recently established in a more general setting of (generic) semi-algebraic generative models~\cite{bendory2023phase}.

\smallskip

Expressing the unknown signal as ${\bf f = \mathcal{G}({\bf z})}$, the problem \eqref{eq: conventional} becomes
\begin{equation}\label{eq: generative}
\min_{\bf z} \|\mathcal{A}\circ \mathcal{G}({\bf z})-{\bf y}\|_2^2.
\end{equation}
The rationale here is that the conditioning of $\mathcal{A}\circ\mathcal{G}$ is more favorable than that of $\mathcal{A}$, at the expense of introducing a bias in the reconstruction. It has indeed been observed in numerical experiments that for high signal-to-noise ratio levels formulation~\eqref{eq: conventional} performs better, while in the case of low signal-to-noise ratio formulation~\eqref{eq: generative} outperforms it \cite{Seifert2024}.

In this paper, we aim to characterize the reconstruction error of \eqref{eq: conventional} and \eqref{eq: generative} in terms of the bias $\min_{\bf z} \|\mathcal{G}({\bf z})- {\bf f}\|_2$ and variance $\Vert\boldsymbol{\varepsilon}\Vert_2$ and propose a practical way to detect if a significant bias has been introduced by the generative model. We furthermore propose a unified variational approach that bridges the two formulations \eqref{eq: conventional} and \eqref{eq: generative} and performs well for different noise levels.

The remaining part of this paper is organized as follows. In Section~\ref{sec: reconstruction error}, we study the recovery properties of the two inverse problem formulations~\eqref{eq: conventional}~and~\eqref{eq: generative} and discuss the role of the bias introduced by a generative model in Section~\ref{sec: bias}. In Section~\ref{sec: unified approach} we propose and evaluate a unified approach and present numerical results supporting our findings in Section~\ref{results}. Finally, Section~\ref{conclusion} contains the discussion of further research directions.

\section{Reconstruction approach}
\label{method}
Let us assume that both the measurement map $\mathcal{A}$ and the generative model $\mathcal{G}$ are injective and bi-Lipschitz with constants $\alpha, \beta\geq 1$. That is, for any ${\bf f}, {\bf f'}\in \mathbb{C}^n$,
\[
\alpha^{-1}\|{\bf f}-{\bf f'}\|_2 \leq \|\mathcal{A}({\bf f}) - \mathcal{A}({\bf f'})\|_2 \leq \alpha \|{\bf f}-{\bf f'}\|_2,
\]
and likewise for $\mathcal{G}$, with $\beta$ in place of $\alpha$.
\begin{remark}
    Note that in the case of the phase retrieval problem, that is, when ${\mathcal{A}({\bf f}) = \vert A {\bf f} \vert^2}$, for a matrix $A\in\mathbb{C}^{m\times n}$, the measurement map $\mathcal{A}$ satisfies the (bi-)Lipschitz property whenever it is injective~\cite{balan2016lipschitz}. However, for most matrices $A$, no bound on the Lipschitz constant $\alpha$ is known. One exception to this is phase retrieval from locally supported measurements~\cite{iwen2019lower}.
\end{remark}

We furthermore assume that the bi-Lipschitz constant $\gamma$ of $\mathcal{A}\circ\mathcal{G}$ is more favorable than that of $\mathcal{A}$, that is, $\gamma < \alpha$. In other words, $\mathcal{G}$ effectively regularities the inverse problem. Moreover, we assume that $\mathcal{G}$ is well-conditioned, in the sense that $0 < \beta - 1 \ll 1$.

\subsection{Characterizing the reconstruction error}\label{sec: reconstruction error}
We denote by ${\bf f_0}$ the ground-truth signal we aim to recover from the measurements
\[
{\bf y} = \mathcal{A}({\bf f_0}) + \boldsymbol{\varepsilon},
\]
where $\boldsymbol{\varepsilon}$ is a bounded noise term. We can now formulate the following results regarding the reconstructing errors resulting from~\eqref{eq: conventional} and~\eqref{eq: generative}.

\begin{lemma}[]
Let ${\bf \tilde{f}} = \argmin_{{\bf f}} \|\mathcal{A}({\bf f})-{\bf y}\|_2$ with ${\bf y} = \mathcal{A}({\bf f_0}) + \boldsymbol{\varepsilon}$. Then the reconstruction error is given by 
\[
\|{\bf \tilde{f}} - {\bf f_0}\|_2 \leq 2\alpha \|\boldsymbol{\varepsilon}\|_2.
\]
\end{lemma}
\begin{proof}
First, note that for any ${\bf f}\in \mathbb{C}^n$,
\[
\|\mathcal{A}({\bf \tilde{f}}) - {\bf y}\|_2 \leq \|\mathcal{A}({\bf f}) - {\bf y}\|_2 \leq \|\mathcal{A}({\bf f})-\mathcal{A}({\bf f_0})\|_2 + \|\boldsymbol{\varepsilon}\|_2.
\]
By letting ${\bf f} = \bf{f_0}$, we get
\[
\|\mathcal{A}({\bf \tilde{f}}) - {\bf y}\|_2 \leq \|\boldsymbol{\varepsilon}\|_2.
\]
Now, using the reverse triangle inequality and the bi-Lipschitz property of $\mathcal{A}$, we obtain
\begin{equation*}
\begin{split}
\|\mathcal{A}({\bf \tilde{f}}) - {\bf y}\|_2 & = \|\mathcal{A}({\bf \tilde{f}}) - \mathcal{A}({\bf f_0}) - \boldsymbol{\varepsilon}\|_2 \\& \geq \|\mathcal{A}({\bf \tilde{f}})-\mathcal{A}({\bf f_0})\|_2-\|\boldsymbol{\varepsilon}\|_2 \\& \geq \alpha^{-1}\|{\bf \tilde{f}}-{\bf f_0}\|_2 - \|\boldsymbol{\varepsilon}\|_2.
\end{split}
\end{equation*}
Combining the two inequalities yields
\[
\alpha^{-1}\|{\bf \tilde{f}}-{\bf f_0}\|_2 \leq 2\|\boldsymbol{\varepsilon}\|_2,
\]
which gives the desired result.
\end{proof}

\begin{lemma}\label{lem: generative bound}
Define ${\bf \tilde{z}} = \argmin_{\bf z} \|\mathcal{A}\circ\mathcal{G}({\bf z}) - {\bf y}\|_2$ with ${{\bf y} = \mathcal{A}({\bf f_0})+\boldsymbol{\varepsilon}}$, and ${\bf\tilde{f}}=\mathcal{G}({\bf\tilde{z}})$. Then the reconstruction error is bounded by 
\[
\|{\bf \tilde{f}} - {\bf f_0}\|_2 \leq (1 + 2\alpha\beta\gamma) \|\mathcal{G}({\bf z_0}) - {\bf f_0}\|_2 + 2\beta\gamma\|\boldsymbol{\varepsilon}\|_2,
\]
where ${\bf z_0} = \argmin_{\bf z} \|\mathcal{G}({\bf z})- {\bf f_0}\|_2$.
\end{lemma}

\begin{proof}
The first step is to upper bound the residual in terms of the bias and variance. By definition of ${\bf y}$ and using the triangular inequality, we have that for any ${\bf z}\in \mathbb{C}^k$, 
$$\|\mathcal{A}\circ \mathcal{G}({\bf\tilde{z}}) - {\bf y}\|_2 \leq \|\mathcal{A}\circ \mathcal{G}({\bf z})-\mathcal{A}({\bf f_0})\|_2 + \| \boldsymbol{\varepsilon}\|_2.$$
Picking ${\bf z} = {\bf z_0}$ and utilizing the bi-Lipschitz property of $\mathcal{A}$ then yields
$$\|\mathcal{A}\circ \mathcal{G}({\bf\tilde{z}}) - {\bf y}\|_2 \leq \alpha \|\mathcal{G}({\bf z_0})-{\bf f_0}\|_2+\| \boldsymbol{\varepsilon}\|_2.$$

The second step is to lower-bound the residual in terms of the error $\|{\bf \tilde{z}} - {\bf z}\|_2$. We split the residual $\|\mathcal{A}\circ \mathcal{G}({\bf\tilde{z}})-{\bf y}\|_2$ as
$$\|\mathcal{A}\circ \mathcal{G}({\bf \tilde{z}}) - \mathcal{A}\circ \mathcal{G}({\bf z_0}) + \mathcal{A}\circ \mathcal{G}({\bf z_0}) - \mathcal{A}({\bf f_0}) - \boldsymbol{\varepsilon}\|_2,$$
and apply the reverse triangle inequality to obtain
\begin{equation*}
    \begin{split}
        \|\mathcal{A}\circ \mathcal{G}({\bf \tilde{z}}) - {\bf y}\|_2 & \geq \|\mathcal{A}\circ \mathcal{G}({\bf\tilde{z}}) - \mathcal{A}\circ \mathcal{G}({\bf z_0})\|_2  \\& - \|\mathcal{A}\circ \mathcal{G}({\bf z_0}) - \mathcal{A}({\bf f_0})\|_2-\|\boldsymbol{\varepsilon}\|_2.
    \end{split}
\end{equation*}
Using bi-Lipschitz property of $\mathcal{A}\circ \mathcal{G}$ and $\mathcal{A}$ and moving terms around we find
$$\gamma^{-1}\|{\bf \tilde{z}} -{\bf z_0}\|_2 \leq  \|\mathcal{A}\circ \mathcal{G}({\bf \tilde{z}}) - {\bf y}\|_2 +  \alpha\|\mathcal{G}({\bf z_0}) -{\bf f_0}\|_2 + \|\boldsymbol{\varepsilon}\|_2.$$

The third step is to split the error $\|{\bf\tilde{f}} - {\bf f_0}\|_2$ and apply the triangle inequality
\begin{equation*}
    \begin{split}
        \|{\bf\tilde{f}} - {\bf f_0}\|_2 = & \|\mathcal{G}({\bf \tilde{z}})-\mathcal{G}({\bf z_0}) + \mathcal{G}({\bf z_0}) - {\bf f_0}\|_2 \\& \leq \|\mathcal{G}({\bf\tilde{z}}) - \mathcal{G}({\bf z_0})\|_2+ \|\mathcal{G}({\bf z_0}) - {\bf f_0}\|_2.
    \end{split}
\end{equation*}
The bi-Lipschitz property of $\mathcal{G}$ then yields
$$\|{\bf\tilde{f}} - {\bf f_0}\|_2\leq \beta \|{\bf \tilde{z}} - {\bf z_0}\|_2 + \|\mathcal{G}({\bf z_0}) - {\bf f_0}\|_2.$$
The desired result follows by combing the three steps.
\end{proof}

\begin{remark}
    Under the assumptions stated above, it is not unreasonable to assume $\beta\gamma < \alpha$, so that the regularized problem indeed leads to less amplification of noise, at the expense of a bias which mainly depends on the expressiveness of the generative model. 
    More specifically, there is a trade-off between the introduced bias and the noise amplification. More restrictive generative model would have smaller $\beta$, leading to the smaller variance, but at the same time it would lead to large bias term in out-of-distribution scenarios (see Figure~\ref{fig:example1_recon}).
\end{remark}

\subsection{Detecting bias}\label{sec: bias}
Even though using regularization in the form of a generative model allows to significantly reduce the noise amplification in the scenario when the ground truth signal ${\bf f_0} = \mathcal{G}({\bf z_0})$ fits the generative model perfectly, in practice this is often not the case. In many applications, such as material science~\cite{agour2013investigation, lehmann2022characterization}, lithography~\cite{mochi2010actinic}, and circuits board manufacturing~\cite{kaya2023development}, the ground truth signal has the form ${\bf f_0} = \mathcal{G}({\bf z_0}) + {\boldsymbol{\eta}}$, where $\mathcal{G}$ models all the ``perfect'' signals and ${\boldsymbol{\eta}}$ represents the signal imperfections and manufacturing defects. In such a scenario, detecting these defects is an important part of the reconstruction problem. Therefore, given a solution ${\bf \tilde{z}} = \argmin_{\bf z} \|\mathcal{A}\circ\mathcal{G}({\bf z}) - {\bf y}\|_2$ to the optimization problem~\eqref{eq: generative}, we would like to be able to detect and ideally also quantify the bias $\|\mathcal{G}({\bf \tilde{z}})- {\bf f_0}\|_2$. 

Using the result from the previous section, we directly get bounds on the bias in terms of the residual ${\tilde\rho}=\|\mathcal{A}\circ\mathcal{G}({\bf \tilde{z}}) -{\bf y}\|_2$ and the noise level $\sigma=\|\boldsymbol{\varepsilon}\|_2$:
\begin{equation*}
 \alpha^{-1}\left({\tilde\rho} -\sigma\right)  \leq \|\mathcal{G}({\bf \tilde{z}})-{\bf f_0}\|_2  \\ \leq \alpha\left({\tilde\rho}+\sigma\right).
\end{equation*}
However, we can expect to get more refined results by exploiting further properties of the measurement map $\mathcal{A}$. 

\subsection{Towards a unified approach}\label{sec: unified approach}
To minimize the influence of the generative model bias on the reconstruction accuracy, we combine the classical formulation~\eqref{eq: conventional} and the generative formulation~\eqref{eq: generative}, and propose the following unified approach to inference using a generative model
\begin{equation}\label{eq:combined}
\min_{{\bf z},{\bf f}} \|\mathcal{A}({\bf f})-{\bf y}\|_2^2 + \lambda^2 \|\mathcal{G}({\bf z})-{\bf f}\|_2^2.
\end{equation}
Obviously, as $\lambda = 0$, we retrieve \eqref{eq: conventional}, and as $\lambda\rightarrow\infty$, we retrieve \eqref{eq: generative}.

\begin{lemma}\label{lem: unified bound}
The estimate ${\bf \tilde f}$ resulting from \eqref{eq:combined} has a reconstruction error bounded as
\[
\|{\bf \tilde f}-{\bf f_0}\|_2 \leq \lambda\alpha \|G({\bf z_0})-{\bf f_0}\|_2 +  2\alpha\|\boldsymbol{\varepsilon}\|_2
\]
\end{lemma}
\begin{proof}
As in the previous results, the bi-Lipschitz property of the map $\mathcal{A}$ yields
\[
\alpha^{-1}\|{\bf \tilde f}-{\bf f_0}\|_2 \leq \|\mathcal{A}({\bf \tilde f})-{\bf y}\|_2 + \|\boldsymbol{\varepsilon}\|_2.
\]
Using the optimality of ${\bf \tilde{f}}$, for any ${\bf f}\in \mathbb{C}^n$, ${\bf z}\in \mathbb{C}^k$, we trivially have
\[
\|\mathcal{A}({\bf\tilde f})-{\bf y}\|_2 \leq \sqrt{\|\mathcal{A}({\bf f})-\mathcal{A}({\bf f_0})\|_2^2 + \lambda^2 \|\mathcal{G}({\bf z})-{\bf f}\|_2^2 + \|\boldsymbol{\varepsilon}\|_2^2}.
\]
Now by picking ${\bf f} = {\bf f_0}$ and ${\bf z}={\bf z_0}\equiv\min_{\bf z} \|\mathcal{G}({\bf z})-{\bf f_0}\|$ and bounding the expression above using $\sqrt{a^2+b^2} \leq |a|+|b|$, we get the desired result.
\end{proof}

\begin{remark}
To ensure optimal reconstruction, the regularization parameter $\lambda$ should be adaptively selected based on the noise intensity $\|\boldsymbol{\varepsilon}\|_2$.
If we let $\lambda \propto \|\boldsymbol{\varepsilon}\|_2$, the reconstruction error of the combined method is bounded as
\[
\|{\bf\tilde f}-{\bf f_0}\|_2 \leq C \|\boldsymbol{\varepsilon}\|_2,
\]
where the constant $C$ may depend on the bias $\|\mathcal{G}({\bf z_0})-{\bf f_0}\|_2$. In the regime when $\|\boldsymbol{\varepsilon}\|_2 \rightarrow 0$, this gives us the desired asymptotic result. It is not clear if this is also the case in the high noise intensity regime  $\|\boldsymbol{\varepsilon}\|_2 \rightarrow \infty$, that is, if for lower signal-to-noise ratio, $\lambda \propto \|\boldsymbol{\varepsilon}\|_2$ increases fast enough to activate the generative model.

In practice, however, the noise intensity is not known in advance. In this case, one possible solution is to let the parameter $\lambda$ vary throughout the iterative reconstruction process, depending on the size of the data fitting term $\|\mathcal{A}({\bf f})-{\bf y}\|_2^2$.
\end{remark}
\section{Numerical results}
\label{results}
One can view \eqref{eq: conventional}, \eqref{eq: generative}, and \eqref{eq:combined} as partial cases of the generic non-linear optimization problem of the form
\begin{equation}\label{eq: unified formulation}
\min_{{\bf x}\in \mathbb{C}^d} \|\mathcal{A}\circ \mathcal{B}({\bf x}) - {\bf y}\|_2^2 + \lambda^2 \|{\bf w}\odot {\bf x}\|_2^2.
\end{equation}
We retrieve the specific instances as follows:
\begin{itemize}
\item \eqref{eq: conventional} by letting $d = n$, ${\bf x}\coloneq {\bf f}$, $\mathcal{B} = \mathcal{I}$, and $\lambda = 0$;
\item \eqref{eq: generative} by letting $d = k$, ${\bf x}\coloneq {\bf z}$, $\mathcal{B} = \mathcal{G}$, and $\lambda = 0$;
\item \eqref{eq:combined} by letting $d = k + n$ with ${\bf x} = ({\bf x_1}, {\bf x_2})$, where ${\bf x_1}\in \mathbb{C}^k$ and ${\bf x_2}\in \mathbb{C}^n$. We define $\mathcal{B}:\mathbb{C}^{n+k} \rightarrow \mathbb{C}^n$ as $\mathcal{B}({\bf x_1}, {\bf x_2}) = \mathcal{G}({\bf x_1}) + {\bf x_2}$, and $w\in \mathbb{C}^{k + n}$ is defined by $w(t) = 0$ for $t\in [k]$ and $w(t) = 1$ for $t\in [k + n]\setminus [k]$. The optimization problem~\eqref{eq:combined} then follows by taking ${\bf x_1} \coloneq {\bf z}$ and ${\bf x_2} \coloneq {\bf f} - \mathcal{G}({\bf z})$.
\end{itemize}

To obtain the object reconstruction ${{\bf\tilde f} = \mathcal{B}({\bf\tilde z})}$, we solve~\eqref{eq: unified formulation} using a Quasi-Newton method, such as limited-memory Broyden-Fletcher-Goldfarb-Shanno (L-BFGS) algorithm~\cite{liu1989limited}.
The reconstruction algorithm we employ thus has parameters $\mathcal{B}, \lambda, {\bf w}$, and a stopping tolerance for L-BFGS.

\subsection{Experiment}
For our numerical experiment, we define the measurement map as
\[
\mathcal{A}({\bf f}) = |A{\bf f}|^2,
\]
where matrix $A$ represents a masked Fourier transform, corresponding to the measurements with $\ell$ different probes 
\[
A = \left(\begin{matrix} F\text{diag}({\bf a}_1) \\ F\text{diag}({\bf a}_2) \\ \vdots \\ F\text{diag}({\bf a}_\ell)\end{matrix}\right).
\]
Here, $F\in\mathbb{C}^{n\times n}$ is the discrete Fourier transform matrix and ${\bf a}_i \in \mathbb{R}^{n}$ are random binary probes, that is, their entries are independent identically distributed Bernoulli random variables. The number of measurements is thus $m=n\cdot\ell$.

We define the generative model as
\[
\mathcal{G}({\bf z}) = G{\bf z} + {\bf b},
\]
where ${\bf b}\in\mathbb{C}^n$ and $G\in\mathbb{R}^{n\times k}$
are obtained by principle component analysis of a data set of handwritten digits \cite{deng2012mnist}. The elements of the data set are $8\times 8$ images, so that $n=64$. For our experiment, we choose $k = 30$.

\begin{figure}
\centering
\includegraphics[scale=.45]{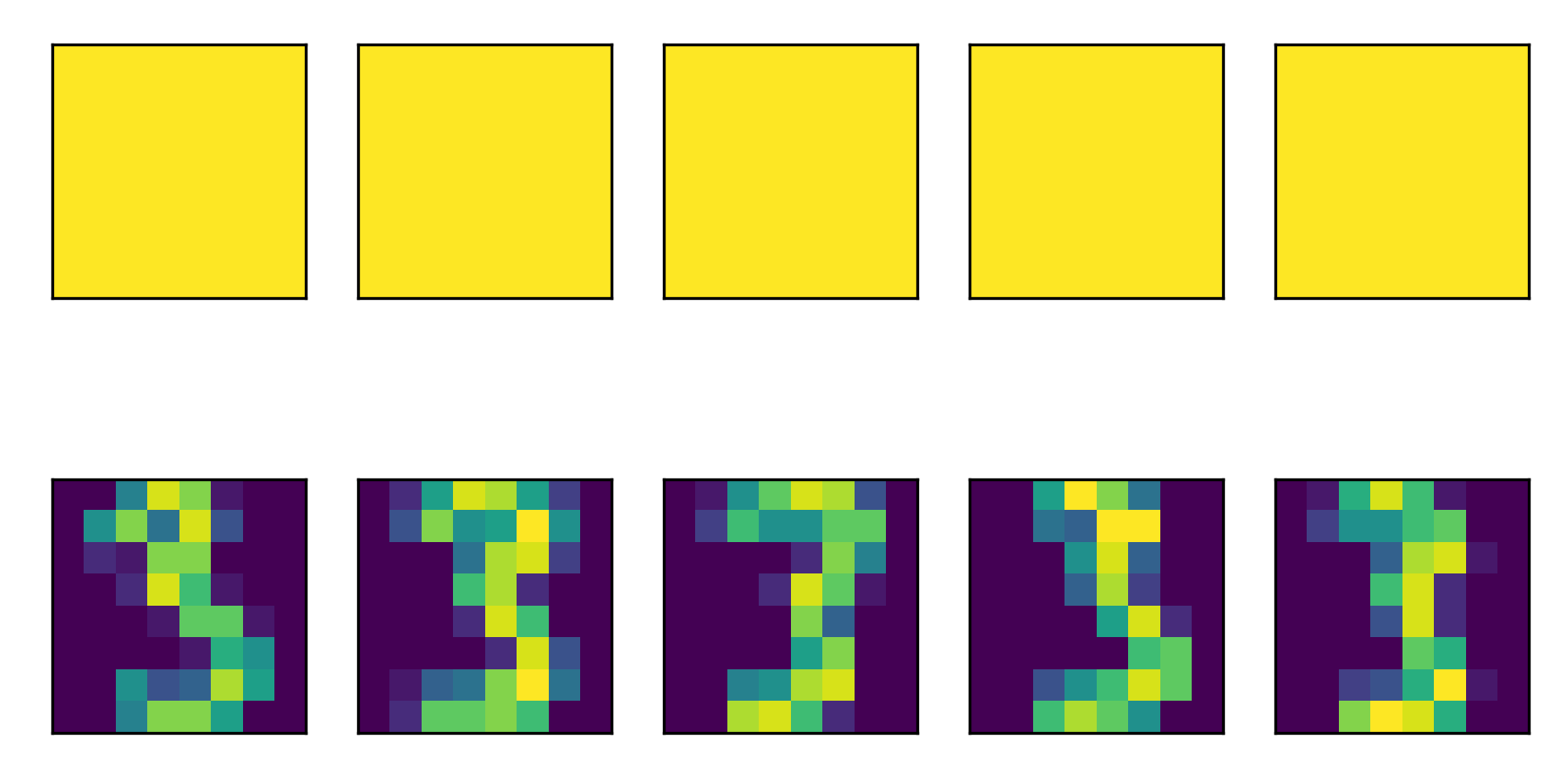}
\caption{Samples of the data set on which the generative model was trained ($n=64$). The top row displays the real part while the bottom row displays the imaginary part.}
\label{fig:example1_data}
\end{figure}
\begin{figure}
\centering
\includegraphics[scale=.45]{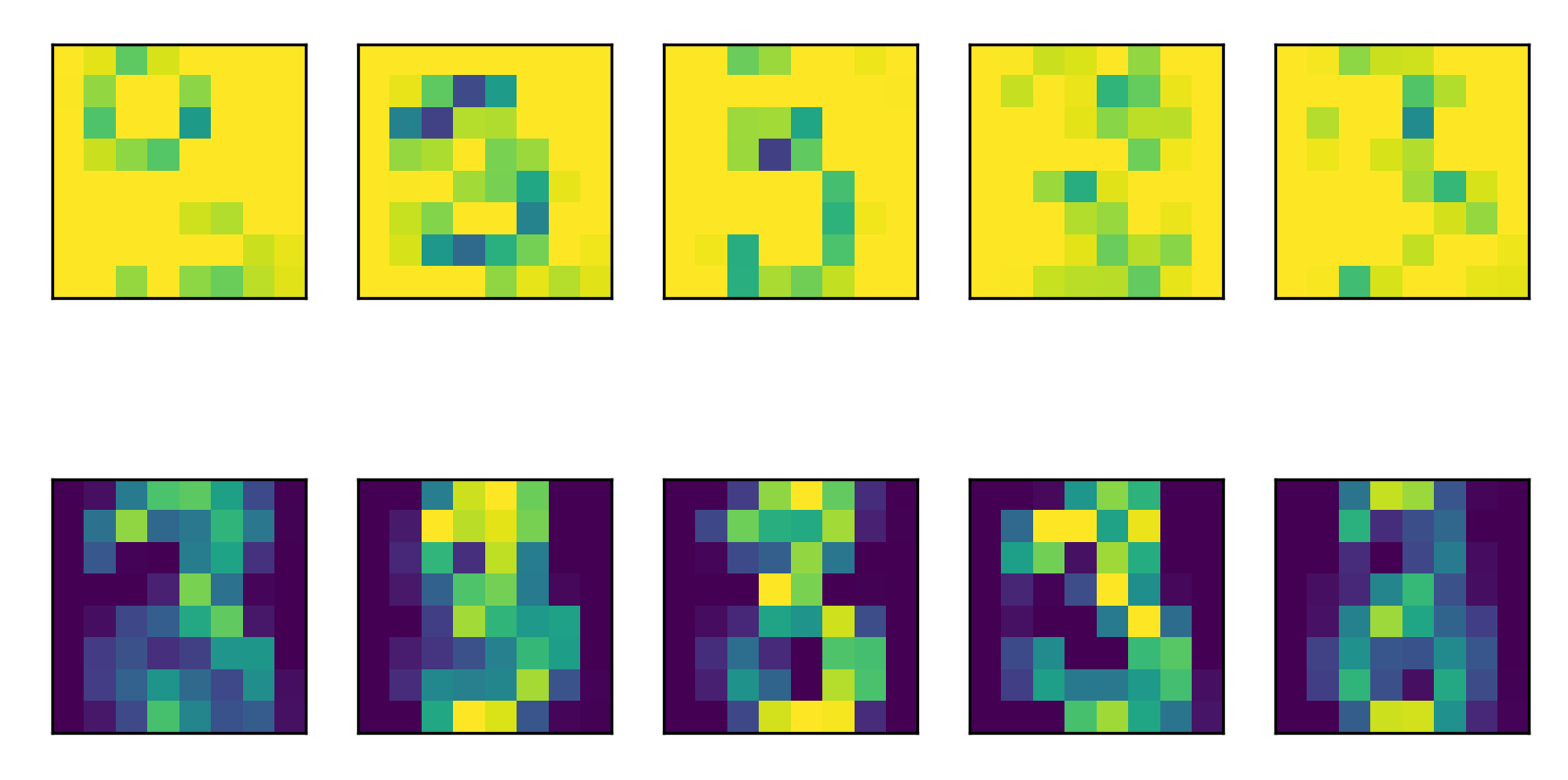}
\caption{Samples generated by the generative model ($k=30$) The top row displays the real part while the bottom row displays the imaginary part.}
\label{fig:example1_generative}
\end{figure}

Examples of the elements of the data set on which the generative model $\mathcal{G}$ is trained are shown in Figure~\ref{fig:example1_data}. Figure~\ref{fig:example1_generative} shows examples of the signals $\mathcal{G}({\bf z})$, ${{\bf z}\in \mathbb{C}^{30}}$, obtained from the trained generative model.  To define the measurement map $\mathcal{A}$, we use $\ell=100$ randomly generated binary probes ${\bf a_j}$, $j\in [\ell]$. 
For the numerical experiment, we generate measurements with additive Gaussian noise $\boldsymbol{\varepsilon}\sim \mathcal{N}(0, \sigma^2 I_m)$. 

To make the reconstruction methods~\eqref{eq: conventional} and~\eqref{eq: generative} more stable to the measurement noise, we introduced additional Tikhonov regularization term by setting $\lambda = \sigma^2$ and ${\bf w} = {\bf 1}$ in~\eqref{eq: unified formulation}. For the reconstruction method~\eqref{eq:combined}, we set $\lambda = 10\cdot\sigma^2$ and  ${\bf w}={\bf 1}$, which also introduces additional Tikhonov regularization term for ${\bf z}$. The reconstruction errors for the three methods, tested on both in-distribution (that is, the ground truth is generated by the generative model) and out-of-distribution (where we use samples from the original data set as ground truth) data are shown in Figure~\ref{fig:example1_recon}.

We see that for high signal-to-noise ratio levels, all methods perform well on in-distribution data, and that for low signal-to-noise ratio levels the generative model shows a slight advantage. On out-of-distribution data, the generative approach~\eqref{eq: generative} clearly shows the bias in the error for high signal-to-noise ratio regime. The combined method~\eqref{eq:combined} achieves the best result for both low and high signal-to-noise ratio levels.
\begin{figure}
\begin{tabular}{c}
\includegraphics[scale=.5]{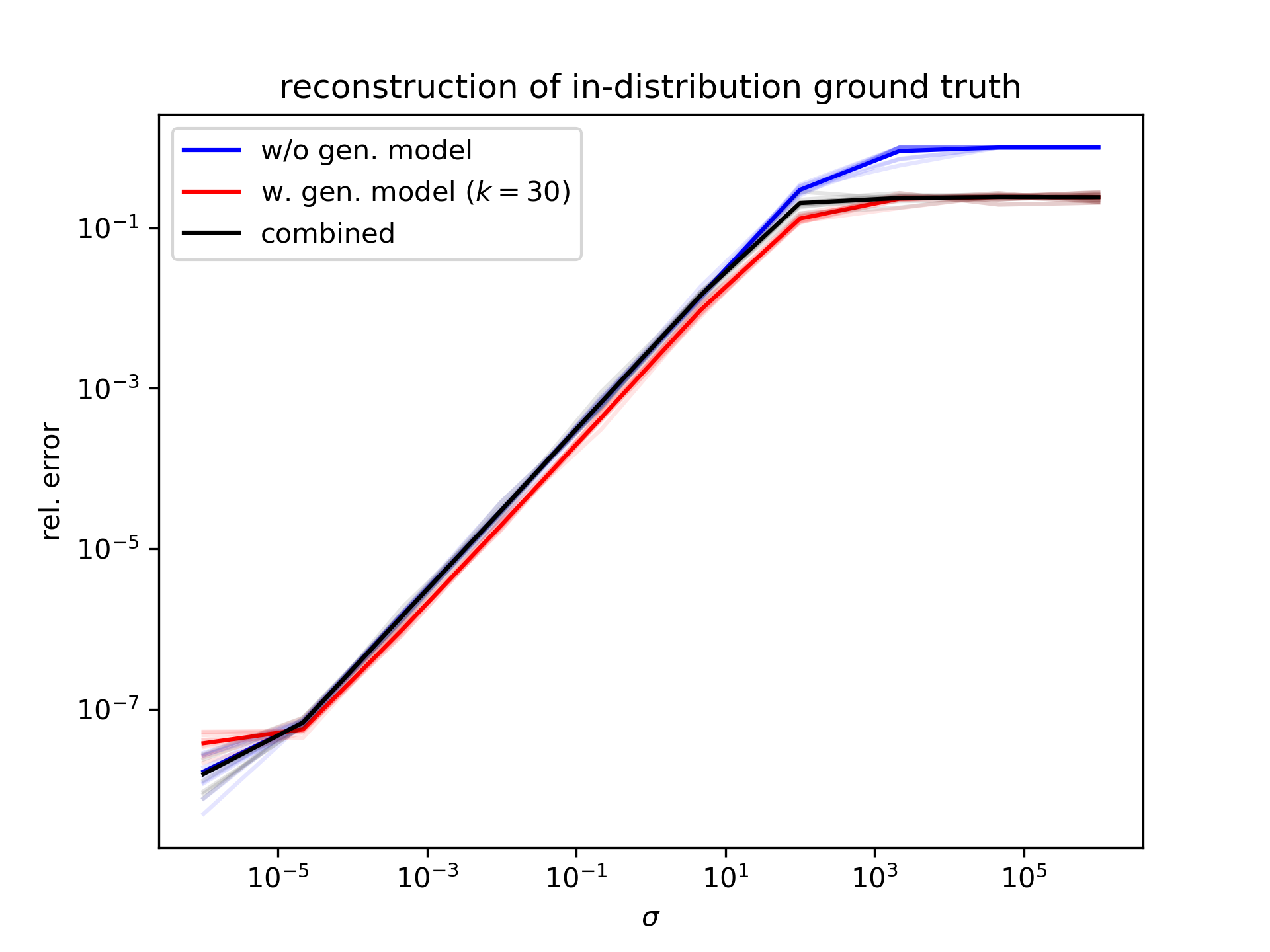}\\
\includegraphics[scale=.5]{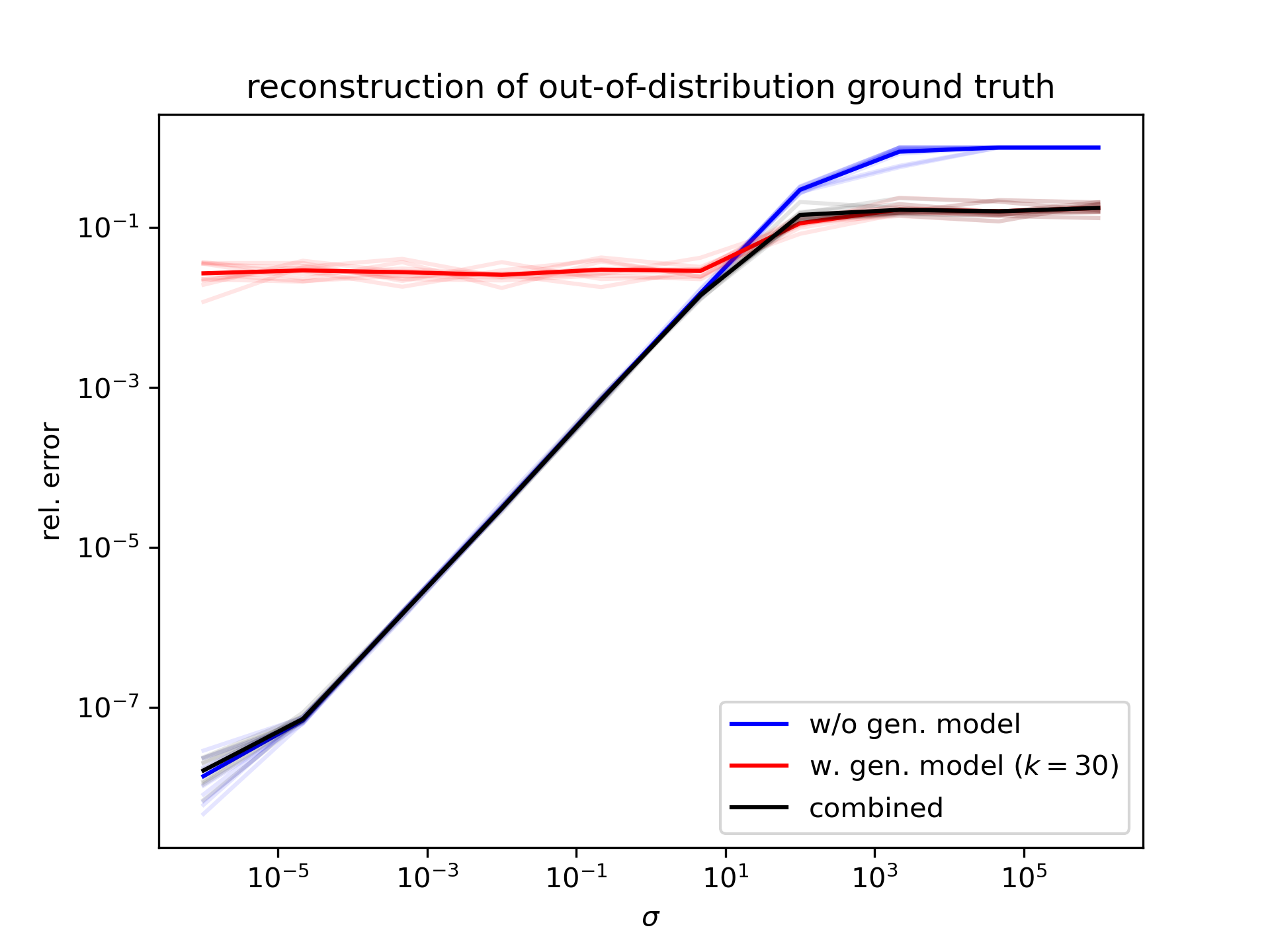}
\end{tabular}
\caption{Relative reconstruction error for the three methods for varying signal-to-noise ratio levels, tested on both in-distribution and out-of-distribution scenarios.}
\label{fig:example1_recon}
\end{figure}

\section{Conclusion and discussion}
\label{conclusion}
In this paper, we explored the use of generative models to regularize certain inverse problems, such as phase retrieval. These preliminary results we showed indicate that generative priors can indeed improve the robustness of the inverse problem solution to measurement noise, at the expense of introducing a bias in the reconstruction. To mitigate this issue, we propose a method that aims to combine the best characteristics of both conventional and regularized methods by interpolating between them. Numerical results on phase retrieval from masked Fourier measurements show that the combined method can indeed achieve the best results.
However, the presented error bounds are rather crude and can probably be improved with more careful analysis methods. Further research is needed to solidify our understanding of the combined method, to refine it, and to make it feasible for high-dimensional problems.

\bibliographystyle{plain}
\bibliography{SampTA_PtyGenography}

\end{document}